\documentclass[10pt, a4paper, logo, twocolumn, nonumbering]{deepmind}

\usepackage[utf8]{inputenc}
\usepackage[T1]{fontenc}
\usepackage[draft]{hyperref}
\usepackage{amsfonts}
\usepackage{amssymb}
\usepackage{amsmath}
\usepackage{amsthm}
\usepackage{nicefrac}
\usepackage{microtype}
\usepackage{graphicx}
\usepackage{algorithm, algorithmic}
\usepackage{mathpazo}
\usepackage{xcolor}
\usepackage[authoryear, round]{natbib} 

\def\w{\mathbf{w}}
\def\x{\mathbf{x}}

\def\g{\gamma}
\def\A{\mathbf{A}}
\def\b{\mathbf{b}}
\def\I{\mathbf{I}}
\def\X{\mathbf{X}}
\def\P{\mathbf{P}}

\def\D{\mathbf{D}}

\def\tr{^\top}
\def\E#1{\mathbb{E}\left[#1\right]}
\newtheorem{proposition}{Proposition}

\title{When to use parametric models in reinforcement learning?}

\author{Hado van Hasselt, Matteo Hessel, John Aslanides\\DeepMind, London, UK}

\begin{abstract}
We examine the question of when and how parametric models are most useful in reinforcement learning.  In particular, we look at commonalities and differences between parametric models and experience replay.  Replay-based learning algorithms share important traits with model-based approaches, including the ability to \textit{plan}: to use more computation without additional data to improve predictions and behaviour. We discuss when to expect benefits from either approach, and interpret prior work in this context. We hypothesise that, under suitable conditions, replay-based algorithms should be competitive to or better than model-based algorithms if the model is used only to generate fictional transitions from observed states for an update rule that is otherwise model-free. We validated this hypothesis on Atari 2600 video games. The replay-based algorithm attained state-of-the-art data efficiency, improving over prior results with parametric models.
\end{abstract}

\begin{document}
\maketitle

The general setting we consider is learning to make decisions from finite interactions with an environment.  Although the distinction is not fully unambiguous, there exist two prototypical classes of algorithms: those that learn without an explicit model of the environment (\emph{model free}), and those that first learn a model and then use it to plan a solution (\emph{model based}).

There are good reasons for building the capability to learn some sort of model of the world into artificial agents.  Models may allow transfer of knowledge in ways that policies and scalar value predictions do not, and may allow agents to acquire rich knowledge about the world before knowing how this knowledge is best used. In addition, models can be used to \emph{plan}: to use additional computation, without requiring additional experience, to improve the agent's predictions and decisions.

In this paper, we discuss commonalities and differences between \emph{parametric models} and \emph{experience replay} \citep{Lin:1992}. Although replay-based agents are not always thought of as model-based, replay shares many characteristics that we often associate with parametric models.  In particular, we can `plan' with the experience stored in the replay memory in the sense that we can use additional computation to improve the agent's predictions and policies in between interactions with the real environment.

Our work was partially inspired by recent work by \citet{Kaiser:2019}, who showed that planning with a parametric model allows for data-efficient learning on several Atari video games. A main comparison was to Rainbow DQN \citep{Hessel:2018}, which uses replay. We explain why their results may perhaps be considered surprising, and show that in a like-for-like comparison Rainbow DQN outperformed the scores of the model-based agent, with less experience and computation.

We discuss this in the context of a broad discussion of parametric models and experience replay. We examine equivalences between them, potential failure modes of planning with parametric models, and how to exploit parametric models in addition to, or instead of, using them to provide imagined experiences to an otherwise model-free algorithm.

\def\env{\mathcal{E}}
\def\batch{\mathcal{B}}
\def\s{\tilde{s}}
\def\a{\tilde{a}}
\def\r{\tilde{r}}
\begin{algorithm}[t]
\begin{algorithmic}[1]
\STATE Input: state sample procedure $d$
\STATE Input: model $m$
\STATE Input: policy $\pi$
\STATE Input: predictions $v$
\STATE Input: environment $\env$
\STATE Get initial state $s \gets \env$
\FOR{iteration $ \in \{1, 2, \ldots, K\}$}
\FOR{interaction $ \in \{1, 2, \ldots, M\}$}
\STATE Generate action: $a \gets \pi(s)$
\STATE Generate reward, next state: $r, s' \gets \env(a)$
\STATE $m, d \gets$ {\sc UpdateModel}($s, a, r, s'$) \label{model_update}
\STATE $\pi, v \gets$ {\sc UpdateAgent}($s, a, r, s'$) \label{real_policy_update}
\STATE Update current state: $s \gets s'$
\ENDFOR
\FOR{planning step $ \in \{1, 2, \ldots, P\}$}
\STATE Generate state, action $\s, \a \gets d$ \label{sample_state}
\STATE Generate reward, next state: $\r, \s' \gets m(\s, \a)$ \label{sample_model}
\STATE $\pi, v \gets$ {\sc UpdateAgent}($\s, \a, \r, \s'$) \label{model_policy_update}
\ENDFOR
\ENDFOR
\end{algorithmic}
\caption{\label{alg:mbrl} Model-based reinforcement learning}
\end{algorithm}

\section{Model-based reinforcement learning}

We now define the terminology that we use in the paper, and present a generic algorithm that encompasses both model-based and replay-based algorithms.

We consider the reinforcement learning setting \citep{SuttonBarto:2018} in which an \emph{agent} interacts with an \emph{environment}, in the sense that the agent outputs actions and then obtains observations and rewards from the environment. We consider the \emph{control} setting, in which the goal is to optimise the accumulation of the rewards over time by picking appropriate sequences of actions. The action an agent outputs typically depends on its \emph{state}. This state is a function of past observations; in some cases it is sufficient to just use the immediate observation as state, in other cases a more sophisticated agent state is required to yield suitable decisions.

We use the word \emph{planning} to refer to any algorithm that uses additional computation to improve its predictions or behaviour without consuming additional data.  Conversely, we reserve the term \emph{learning} for updates that depend on newly observed experience.

The term \emph{model} will refer to functions that take a state and action as input, and that output a reward and next state. Sometimes we may have a perfect model, as in board games (e.g., chess and go); sometimes the model needs to be learnt before it can be used. Models can be stochastic, to approximate inherently stochastic transition dynamics, or to model the agent's uncertainty about the future. Expectation models are deterministic, and output (an approximation of) the expected reward and state.  If the true dynamics are stochastic, iterating expectation models multiple steps may be unhelpful, as an expected state may itself not be a valid state; the output of a model may not have useful semantics when using an expected state as input rather than a real state \citep[cf.][]{Wan:2019}.  Planning is associated with models, because a common way to use computation to improve predictions and policies is to search using a model. For instance, in Dyna \citep{Sutton:1990}, learning and planning are combined by using new experience to learn both the model and the agent's predictions, and then planning to further improve the predictions.

\textit{Experience replay} \citep{Lin:1992} refers to storing previously observed transitions to replay later for additional updates to the predictions and policy. Replay may also be used for planning and, when queried at state-action pairs we have observed, experience replay may be indistinguishable from an accurate deterministic model. Sometimes, there may be no practical differences between replay and models, depending on how they are used. On the other hand, a replay memory is less flexible than a model, since we cannot query it at arbitrary states that are not present in the replay memory.

\subsection{A generic algorithm}

Algorithm \ref{alg:mbrl} is a generic model-based learning algorithm. It runs for $K$ iterations, in each of which $M$ interactions with the environment occur. The total number of interactions is thus $T \equiv K \times M$. The experience is used to update a model (line \ref{model_update}) and the policy or predictions of the agent (line \ref{real_policy_update}).  Then, $P$ steps of `planning' are performed, where transitions sampled from the model are used to update the agent (line \ref{model_policy_update}).  For $P = 0$, the model is not used, hence the algorithm is model-free (we could then also skip line \ref{model_update}). If $P > 0$, and the agent update in line \ref{real_policy_update} does not do anything, we have a purely model-based algorithm.  The agent updates in lines \ref{real_policy_update} and \ref{model_policy_update} could differ, or they could treat real and modelled transitions equivalently.

Many known algorithms from the model-based literature are instances of algorithm \ref{alg:mbrl}. If lines \ref{real_policy_update} and \ref{model_policy_update} both update the agent's predictions in the same way, the resulting algorithm is known as Dyna \citep{Sutton:1990} -- for instance, if predictions $v$ include action values (normally denoted with $q$) and we update using Q-learning \citep{Watkins:1989,WatkinsDayan:1992}, we obtain Dyna-Q \citep{SuttonBarto:2018}. One can extend Algorithm \ref{alg:mbrl} further, for instance by allowing planning and model-free learning to happen simultaneously. Such extensions are orthogonal to our discussion and we do not discuss them further.

Some algorithms typically thought of as being model-free also fit into this framework. For instance, DQN \citep{Mnih:2013,Mnih:2015} and neural-fitted Q-iteration \citep{Riedmiller:2005} match Algorithm \ref{alg:mbrl}, if we stretch the definitions of `model' to include the more limited replay buffers. DQN learns from transitions sampled from a replay buffer by using Q-learning with neural networks.
In Algorithm \ref{alg:mbrl}, this corresponds to updating a non-parametric model, in line \ref{model_update}, by storing observed transitions in the buffer (perhaps overwriting old transitions); line \ref{sample_model} then retrieves a transition from this buffer. The policy is only updated with transitions sampled from the replay buffer (i.e., line \ref{real_policy_update} has no effect).

\section{Model properties}\label{sec:model_properties}

\begin{figure*}[t]
    \centering
    \includegraphics[width=0.248\linewidth]{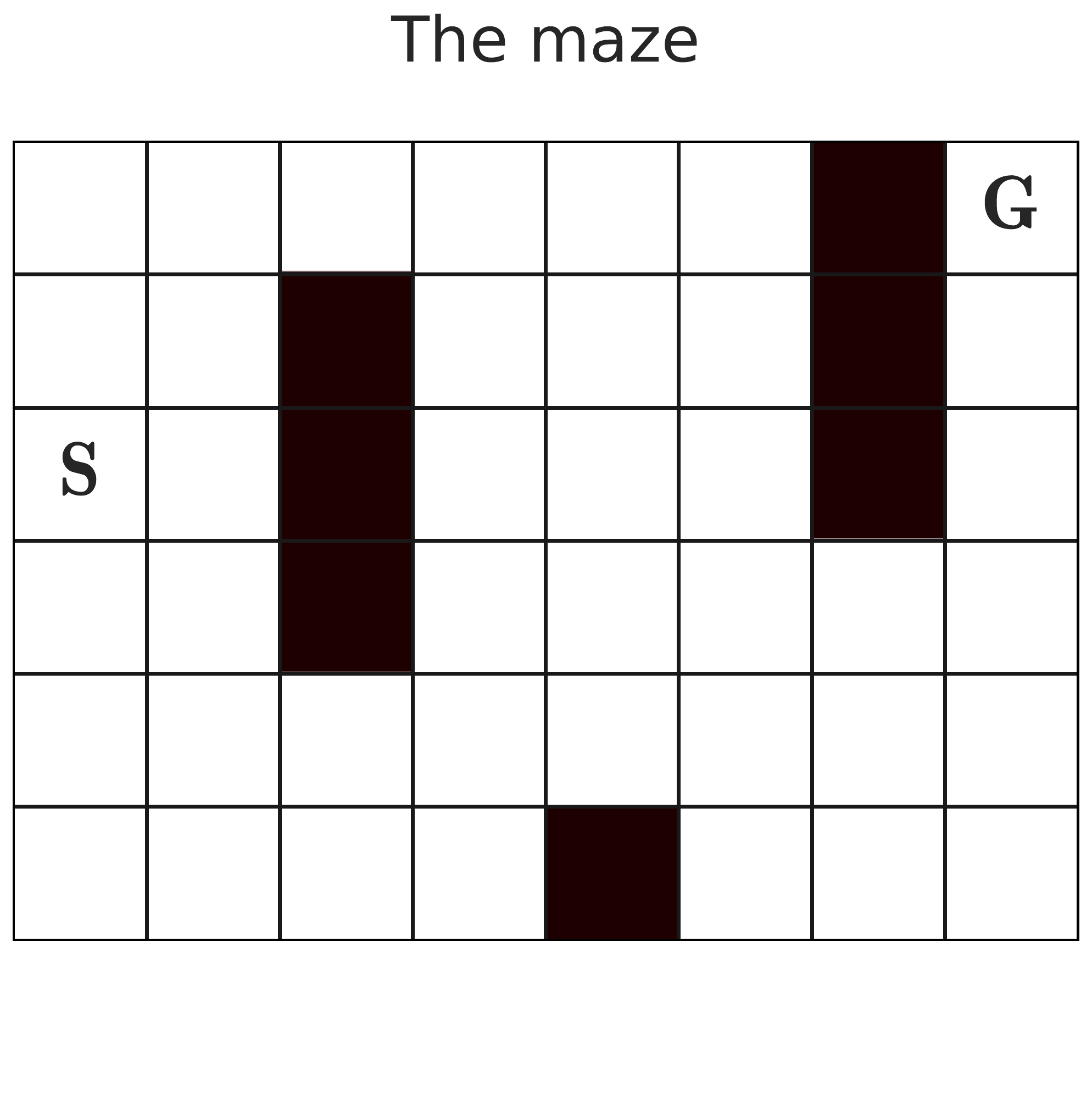}\hspace{.7 cm}
    \includegraphics[width=0.25\linewidth]{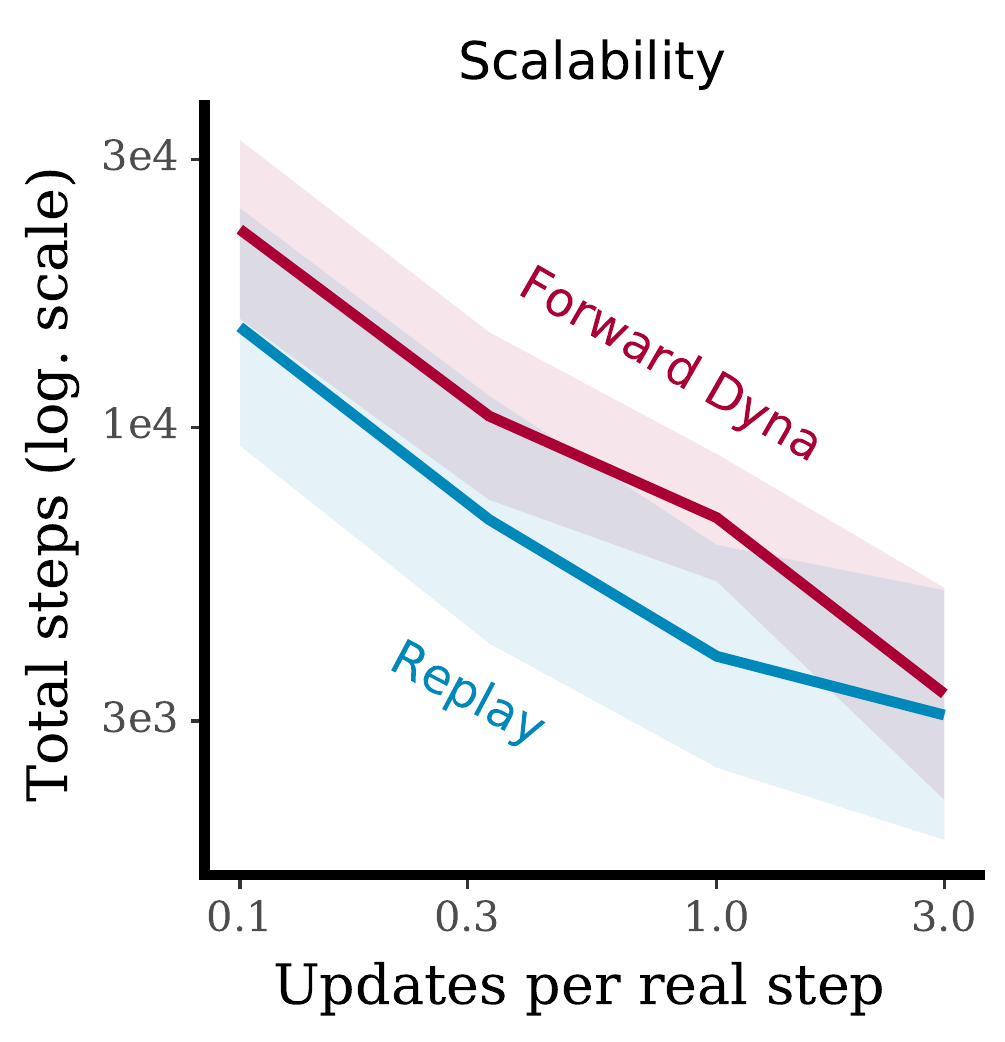}
    \caption{\textbf{Left}: the layout of the grid world \citep{SuttonBarto:2018}, `S' and `G' denote the start and goal state, respectively. \textbf{Right}: Q-learning with replay (blue) or Dyna-Q with a parametric model (red); $y$-axis: the total number of steps to complete 25 episodes of experience, $x$-axis: the number of updates per step in the environment. Both axes are on a logarithmic scale.}
    \label{dyna_maze}
\end{figure*}

A main advantage of using models is the ability to \emph{plan}: to use additional computation, but no new data, to improve the agent's policy or predictions. \citet{SuttonBarto:2018} illustrate the benefits of planning in a simple grid world (Figure \ref{dyna_maze}, on the left), where the agent must learn to navigate along the shortest path to a fixed goal location. On the right of Figure \ref{dyna_maze} we use this domain to show how the performance of a replay-based Q-learning agent (blue) and that of a Dyna-Q agent (red) scale similarly with the amount of planning (measured in terms of the number of updates per real environment step). Both agents use a multi-layer perceptron to approximate action values, but Dyna-Q also used identical networks to model transitions, terminations and rewards. The algorithm is called `forward Dyna' in the figure, because it samples states from the replay and then steps forward one step using the model. Later we will consider a variant that, instead, steps backward with an inverse model. The appendix contains further details on the experiments.

\subsection{Computational properties}

There are clear computational differences between using parametric models and replay. For instance, \citet{Kaiser:2019} use a fairly large deep neural network to model the pixel dynamics in Atari, which means predicting a single transition can require non-trivial computation. In general, parametric models typically require more computations than it takes to sample from a replay buffer.

On the other hand, in a replay model the capacity and memory requirements are tightly coupled: each transition that is stored takes up a certain amount of memory. If we do not remove any transitions, the memory can grow unbounded. If we limit the memory usage, then this implies that the effective capacity of the replay is limited as any transitions we replace are forgotten completely.  In contrast, parametric models may be able to achieve good accuracy with a fixed and comparatively small memory footprint.

\subsection{Equivalences}
Suppose we manage to learn a model that perfectly matches the transitions observed thus far.  If we would then use such a perfect model to generate experiences only from states that were actually observed, the resulting updates would be indistinguishable from doing experience replay. In that sense, replay matches a perfect model, albeit only from the states we have observed.\footnote{One could go one step further and extend replay to be a full non-parametric model as well. For instance \citet{Pan:2018} use kernel methods to define what happens when we query the replay-based model at states that are not stored in the buffer.} Therefore, all else being equal, we would expect that using an imperfect (e.g., parametric) model to generate fictional experiences from truly observed states should probably not result in better learning.

There are some subtleties to this argument.  First, the argument can be made even stronger in some cases. When making linear predictions with least-squares temporal-difference learning \citep[LSTD, ][]{Bradtke:96,Boyan:1999}, the model-free algorithm on the original data does not require (or indeed benefit from) planning: the solution will already be a best fit (in a least squares sense) even with a single pass through the data. In fact, if we fit a linear model to the data and then fully solve this model, the solution turns out to be equal to the LSTD solution \citep{Parr:2008}. One can also show that exhaustive replay with linear TD($\lambda$) \citep{Sutton:1988} is equivalent to a one-time pass through the data with LSTD($\lambda$) \citep{vanSeijen:2015}, because replay similarly allows us to solve the empirical `model' that is implicitly defined by the observed data.

These equivalences are however limited to linear prediction, and do not extend straightforwardly to non-linear functions, or to the control setting.  This leaves open the question of when to use a parametric model rather than replay, or vice versa.

\subsection{When do parametric models help learning?}

\begin{figure*}[t]
    \centering
    \includegraphics[width=0.20\linewidth]{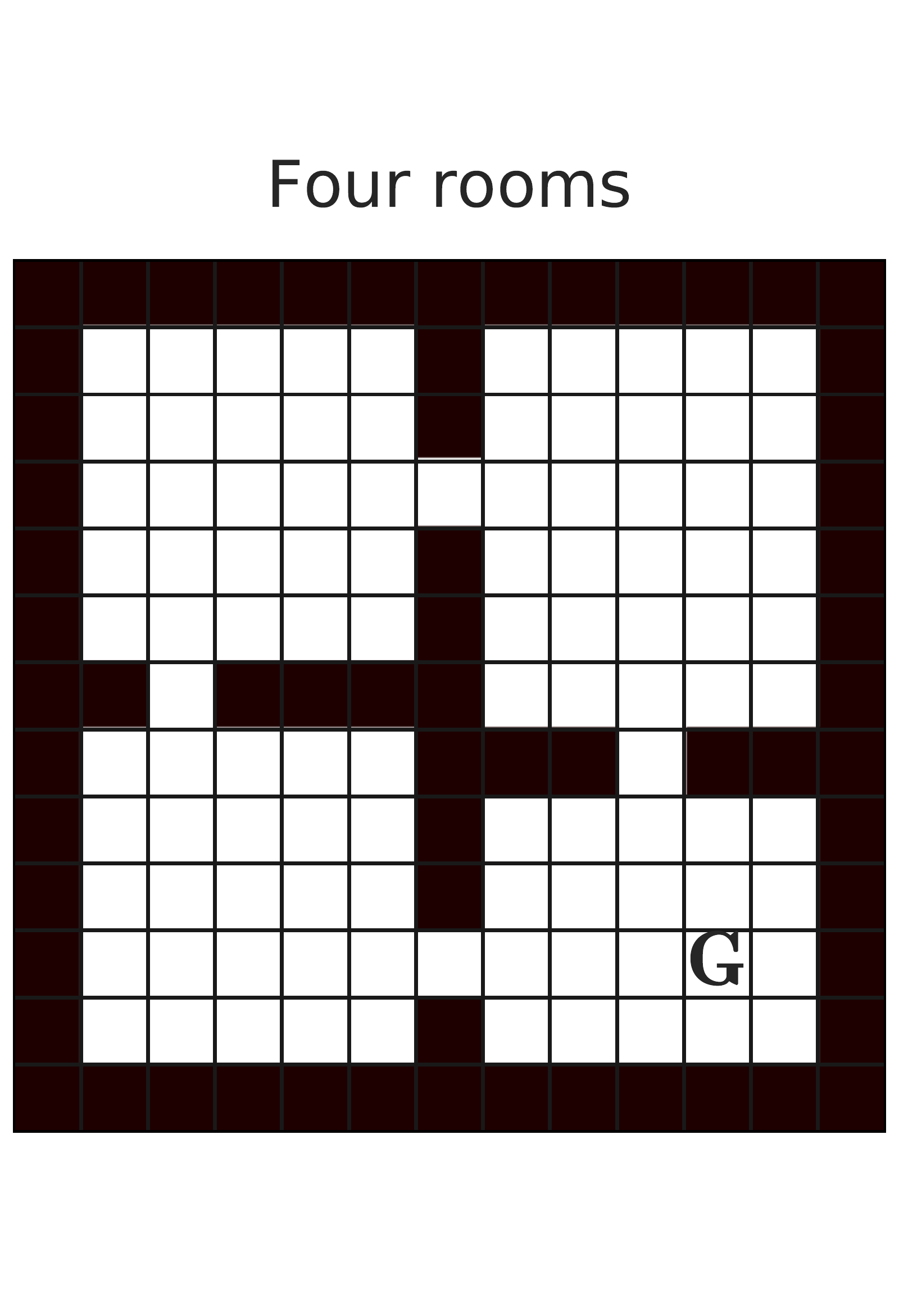}
    \includegraphics[width=0.25\linewidth]{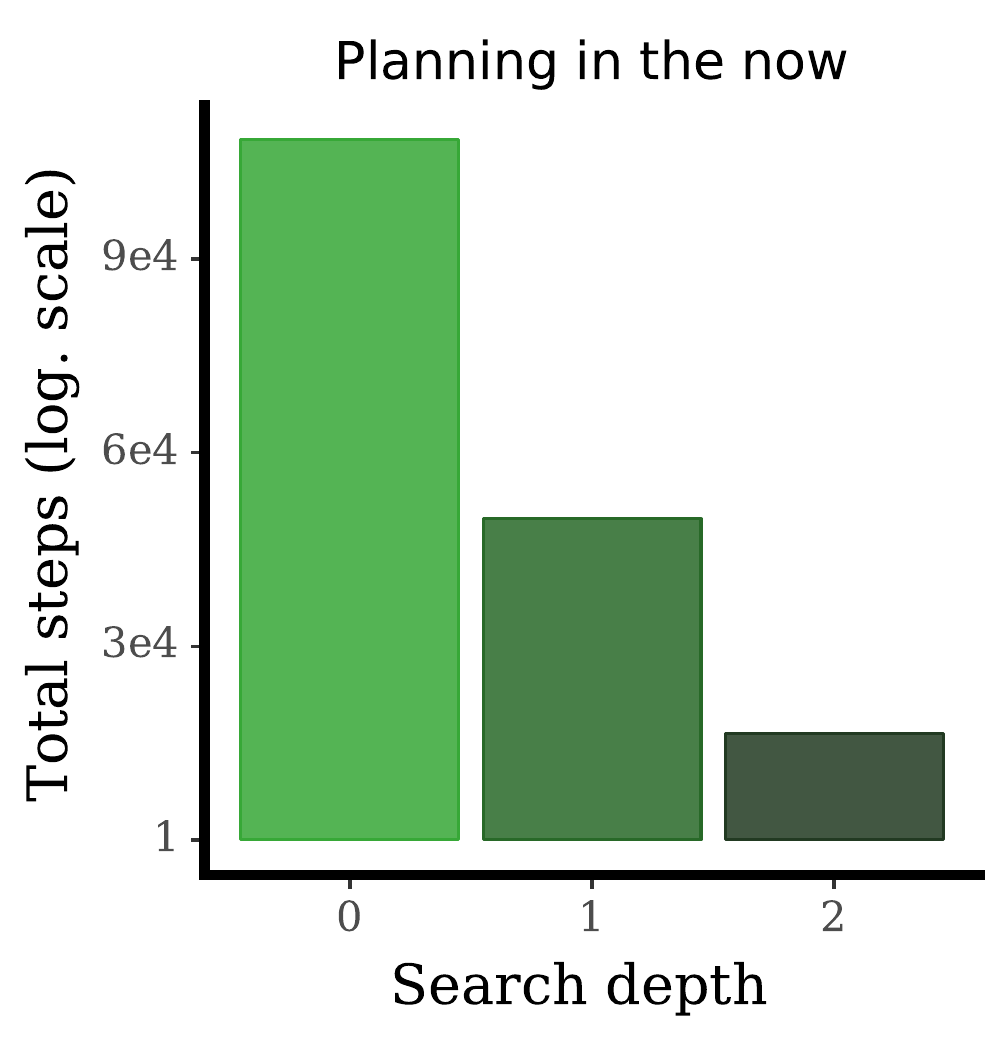}
    \includegraphics[width=0.25\linewidth]{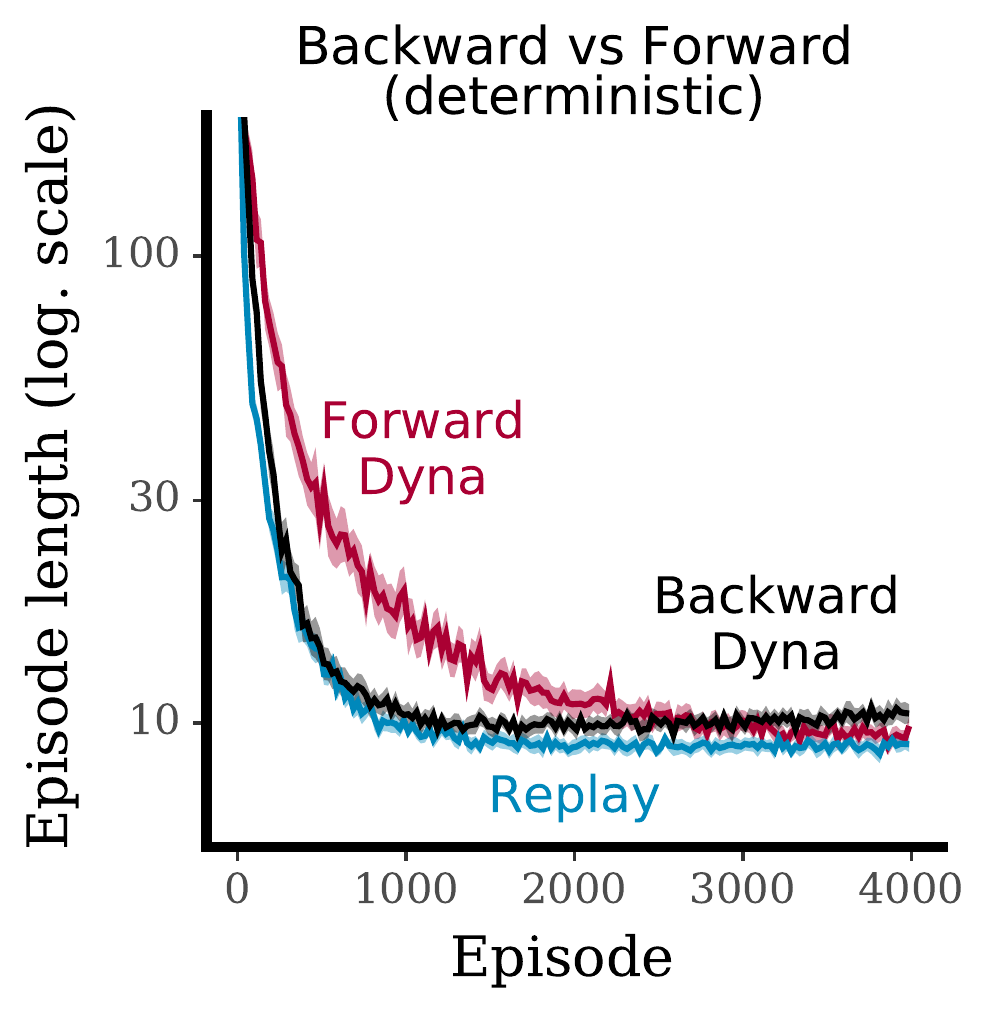}
    \includegraphics[width=0.25\linewidth]{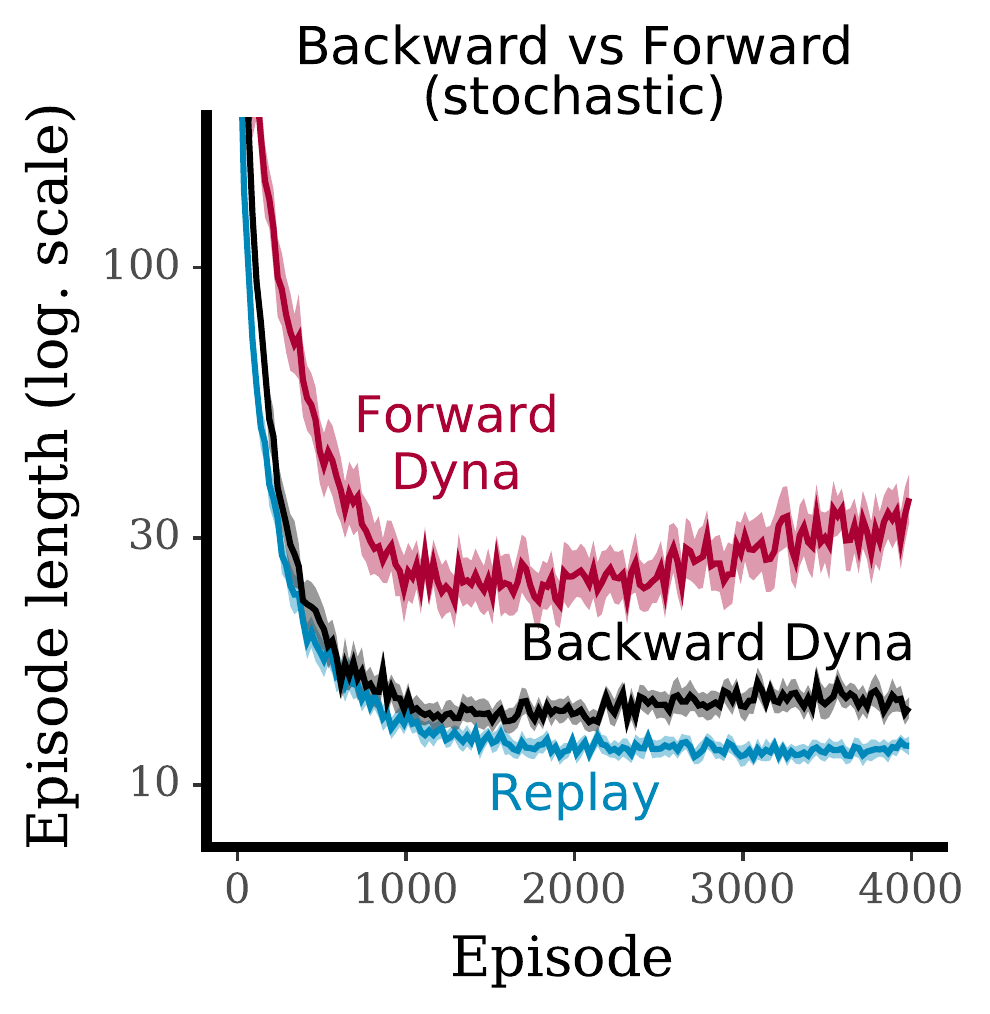}
    \caption{\textbf{Left}: four rooms grid world \citep{Sutton:1998option}. \textbf{Center-left}: planning forward from the current state to update the current behaviour (0 steps corresponds to Q-learning); $y$-axis: total number of steps required to complete 100 episodes, $x$-axis: search depth. \textbf{Center-right}: comparing replay (blue), forward Dyna (red), and backward Dyna (black); $y$-axis: episode length (logarithmic scale), $x$-axis: number of episodes. \textbf{Right}: adding stochasticity to the transition dynamics (in the form of a $20\%$ probability of transitioning to a random adjacent cell irrespectively of the action), then comparing again replay (blue), forward Dyna (red), and backward Dyna (black); $y$-axis: episode length (logarithmic scale), $x$-axis: number of episodes}
    \label{beneficial}
\end{figure*}

When should we expect benefits from learning and using a parametric model, rather than using the actual data?  We discussed important computational differences above.  Here we focus on learning efficiency: when do parametric models help learning?

First, parametric models may be useful to plan into the future to help determine our policy of behaviour. The ability to generalise to unseen or counter-factual transitions can be used to plan from the \emph{current state} into the future (sometimes called planning `in the now' \citep{Kaelbling:2010}), even if this exact state has never before been observed. This is commonly and successfully employed in model-predictive control \citep{Richalet:1978,Morari:1999,Mayne:2014,Wagener:2019}.  Classically, the model is constructed by hand rather than learnt directly from experience, but the principle of planning forward to find suitable behaviour is the same.
It is not possible to replicate this with standard replay, because in interesting rich domains the current state will typically not exactly appear in the replay. Even if it would, replay does not allow easy generation of \emph{possible} next states, in addition to the one trajectory that actually happened.

If we use a model to select actions, rather than trusting its imagined transitions to update the policy or predictions, it may be less essential to have a highly accurate model. For instance, the model may predict a shortcut that does not actually exist; using this to then steer behaviour results in experience that is both suitable to correct the error in the model, and that yields the kind of directed, temporally consistent behaviour typically sought for exploration purposes \citep{lowrey2018plan}. 

We illustrate this with a simple experiment on a classic four room grid-world with deterministic transitions \citep{Sutton:1998option}. We learnt a tabular forward model that generates transitions $(s, a)\rightarrow (s', r, \gamma)$, where $s$ and $s'$ are states, $a$ is an action, $r$ is a reward, and $\gamma \in [0, 1]$ is a discount factor. We then used this model to plan via a simple breadth-first search up to a fixed depth, bootstrapping from a value function $Q(s, a)$ learnt via standard Q-learning. We then use the resulting planned values of the actions at the current state to behave. This process can be interpreted as using a multi-step greedy policy \citep{Efroni:2018} to determine behaviour, instead of the more standard one-step greedy policy. The results are illustrated in the second-to-left plot in Figure \ref{beneficial}: more planning was beneficial.

In addition to planning forward to improve behaviour, models may be useful for credit assignment through \emph{backward} planning.  Consider an algorithm where, as before, we sample real visited states from a replay buffer, but instead of planning one step into the future from these states we plan one step backward. One motivation for this is that if the model is poor then planning a step forward will update the real sampled state with a misleading imagined transition.  This will potentially cause harmful updates to the value at these real states.  Conversely, if we plan backwards we then update an imagined state. If the model is poor this imagined state perhaps does not resemble any real state. Updating such fictional states seems less harmful. 
When the model becomes very accurate, forward and backward planning both start to be equally useful.  For a purely data-driven (partial) model, such as a replay buffer, there is no meaningful distinction. But with a learnt model that is at times inaccurate, backward planning may be less error-prone than forward planning for credit assignment.

We illustrate potential benefits of backward planning with a simple experiment on the four-room environment. In the two right-most plots of Figure \ref{beneficial}, we compare the performance of applying tabular Q-learning to transitions generated by a forward model (red), a backward model (black), or replay (blue). The forward model learns distributions over states, rewards, and terminations $\mathrm{Pr}(r, \gamma, s' \mid s, a)$ and the backward model learns the opposite: $\mathrm{Pr}(s, a \mid r, \gamma, s')$. Both use a Dirichlet(1) prior. We evaluated the algorithms in the deterministic four-room environment, as well as in a stochastic variant where on each step there is a $20\%$ probability of transitioning to a random adjacent cell irrespectively of the action. In both cases, using backward planning resulted in faster learning than forward planning. In the deterministic setting, the forward model catches up later in learning, reaching the same performance of replay after 2000 episodes; instead, planning with a backward model is competitive with replay in early learning but performs slightly worse later in training. We conjecture that the slower convergence in later stages of training may be due to the fact that predicting the source state and action in a transition is a non-stationary problem (as it depends on the agent's policy), and given that early episodes include many more transitions than later ones, it can take many episodes for a Bayesian model to forget the policies observed early in training. The lack of convergence to the optimal policy for the forward planning algorithm in the stochastic setting may instead be due to the independent sampling of the successor state and reward, which may result in inconsistent transitions. Both these issues may be addressed by a suitable choice of the model. More detailed investigations are out of scope for this paper, but it is good to recognise that such modelling choices have measurable effects on learning.

\section{A failure to learn}\label{sec:failure}

We now describe how planning in a Dyna-style learning algorithm can, perhaps surprisingly easily, lead to catastrophic learning updates.

Algorithms that combine function approximation (e.g., neural networks), bootstrapping (as in temporal difference methods \citep{Sutton:1988}), and off-policy learning \citep{SuttonBarto:2018,Precup:2000} can be unstable \citep{Williams:1993,Baird:1995,Sutton:1995,Tsitsiklis:1997,Sutton:2009,Sutton:2016} --- this is sometimes called the \emph{deadly triad} \citep{SuttonBarto:2018,vanHasselt:2018}.

This has implications for Dyna-style learning, as well as for replay methods \citep[cf.][]{vanHasselt:2018}. When using replay it is sometimes relatively straightforward to determine how off-policy the state sampling distribution is, and the sampled transitions will always be real transitions under that distribution (assuming the transition dynamics are stationary).  In contrast, the projected states given by a parametric model may differ from the states that would occur under the real dynamics, due to modelling error. The update rule will then be solving a predictive question for the MDP induced by the model, but with a state distribution that does not match the on-policy distribution in that MDP.

To understand this issue better, consider using Algorithm \ref{alg:mbrl} to estimate expected cumulative discounted rewards $v_\pi(s) = \E{ R_{t+1} + \g R_{t+2} + \ldots \mid S_t = s, \pi}$ for a policy $\pi$ by updating $v_\w(s) \approx v_\pi(s)$ with temporal difference (TD) learning \citep{Sutton:1988}:
\begin{align}
\label{td}
\w & \gets \w + \alpha \delta_t \nabla_{\w} v_\w(S_t) \,,\\
\text{with }  \delta_t &\equiv R_{t+1} + \g_{t+1} v_\w(S_{t+1}) - v_\w(S_t) \,,\notag
\end{align}
where $R_{t+1} \in \mathbb{R}$ and $\g_{t+1} \in [0,1]$ are the reward and discount on the transition from $S_t$ to $S_{t+1}$, and $\alpha > 0$ is a small step size.
Consider linear predictions $v_\w(S_t) = \w\tr\x_t \approx v_{\pi}(S_t)$, where $\x_t \equiv \x(S_t)$ is a feature vector for state $S_t$. The expected TD update is then $\w \gets (\I - \alpha \A)\w + \alpha \b$, with $\b = \E{ R_{t+1} \x_t }$ and $\A = \E{ \x_t\x_t\tr - \g\x_t\x_{t+1} } = \X\tr\D (\I - \g \P\tr) \X$, where the expectation is over the transition dynamics and over the sampling distribution $d$ of the states. The transition dynamics can be written as a matrix $\P$, with $[\P]_{ij} = p(S_{t+1} = i \mid S_t = j,  \pi)$, that maps the vector of all states into the vector of all subsequent states: the  is the probability of transitioning from state $j$ to $i$ under policy $\pi$. The diagonal matrix $\D$ contains the probabilities $[\D]_{ii} = d(i) = P(S_t = i \mid \pi)$ of sampling each state $s$ on its diagonal. The matrix $\X$ contains the feature vectors $\x(s)$ of all states on its rows, and maps between state and feature space. Both $\P$ and $\D$ are linear operators in state space, not feature space.

These updates are guaranteed to be stable (i.e., converge) if $\A = \X\tr\D (\I - \g \P\tr) \X$ is positive semi-definite \citep{Sutton:2016}, with a spectral radius smaller than $1/\alpha$. The deadly triad occurs when $\D$ and $\P$ do not match, as then $\A$ can become negative definite and at which point the spectral radius $\rho(\I - \alpha \A)$ can become larger than one, and the weights can diverge.  This can happen when $\D$ does not correspond to the steady-state distribution of the policy that conditions $\P$ --- that is, if we update off-policy.

\begin{proposition}
Consider uniformly replaying transitions from a buffer containing either full episodes (e.g., add new full episodes on termination, potentially remove an old full episode), and using these transitions in the TD algorithm defined by update \eqref{td}. This algorithm is stable.
\end{proposition}
\begin{proof}
The replay buffer defines an empirical model, where the induced policy is the empirical distribution of actions: $\tilde{\pi}(a|s) = n(s,a)/n(s)$, where $n(s)$ and $n(s, a)$ are the number of times $s$ and the pair $(s, a)$ show up in the replay (the behaviour policy can change while filling the replay, the resulting empirical policy is then a sample of a mixture of these policies). The empirical transitions $[\tilde{\P}]_{ij} = n(i, j)/n(i)$ and state distributions $[\tilde{\D}]_{ii} = n(s)/N$, where $N$ is the total size of the replay buffer, then both correspond to the same empirical policy.  Therefore, $\rho(\tilde{\X}\tr\tilde{\D} (\I - \g \tilde{\P}\tr) \tilde{\X}) > 0$, and TD will be stable and will not diverge.
\end{proof}
This proposition can be extended to the case where transitions are added to the replay one at the time, rather then in full episodes. If, however, we sample states according to a non-uniform distribution (e.g., using prioritised replay) this can make replay-based algorithms less stable and potentially divergent \citep[cf.][]{vanHasselt:2018}.

We now show that a very similar algorithm that uses models in place of replay can diverge.
\begin{proposition}
Consider uniformly replaying states from a replay buffer, then generating transitions with a learnt model $\hat{p}_m$, and using these transitions in a TD update \eqref{td}. This algorithm can diverge.
\end{proposition}
\begin{proof}
The learnt dynamics $\hat{\P}_m \approx \P$ do not necessarily match the empirical dynamics of the replay, which means that the empirical replay distribution $d$, used in the updates, does not necessarily correspond to the steady-state distribution of these dynamics.  Then the model error could lead to a negative definite $\hat{\A} \equiv \X\tr\tilde{\D} (\I - \g \hat{\P}_m\tr) \X$, resulting in a spectral radius $\rho(I - \alpha\hat{\A}) > 1$, and divergence of the parameters $\w$.
\end{proof}

Intuitively, the issue is that the model $m$ can lead to states that are uncommon, or impossible, under the sampling distribution $d$.  Those states are not sampled to be updated directly, but do change through generalisation when sampled states are updated.  This can lead to divergent learning dynamics.

There are ways to mitigate the failure described above. 
First, we could repeatedly iterate the model, and sample transitions \emph{from} the states the model generates as well as \emph{to} those states, to induce a state distribution that is consistent with the model. This is not fully satisfactory, as states typically become ever-more unrealistic when iterating a learnt model. Such fictional updates can also be harmful to learning.
Second, we could rely less on bootstrapping by using multi-step returns \citep{Sutton:1988,vanHasselt:2015,SuttonBarto:2018}. This mitigates the instability \citep[cf.][]{vanHasselt:2018}. In the extreme, full Monte-Carlo updates do not diverge, though they would be high variance.
Third, we could employ algorithms specifically for stable off-policy learning, although these are often specific to the linear setting \citep{Sutton:2008,Sutton:2009,vanHasselt:2014} or assume the sampling is done on trajectory \citep{Sutton:2016}. Note that several algorithms exist that correct the \emph{return} towards a desired policy \citep{Harutyunyan:2016,Munos:2016}, which is a separate issue from off-policy sampling of \emph{states}. Off-policy learning algorithms may be part of the long-term answer, we do not yet have a definitive solution.  To quote \citet{SuttonBarto:2018}:
\textit{
    The potential for off-policy learning remains tantalising, the best way to achieve it still a mystery.
}

Understanding such failures to learn is important to understand and improve our algorithms.  However, just because divergence \emph{can} occur does not mean it \emph{does} occur \citep[cf.][]{vanHasselt:2018}.  Indeed, in the next section we compare a replay-based algorithm to a model-based algorithm which was stable enough to achieve impressive sample-efficiency on the Atari benchmark.

\section{Model-based algorithms at scale}
We now discuss two algorithms in more detail:  first SimPLe \citep{Kaiser:2019}, which uses a parametric model, then Rainbow DQN \citep{Hessel:2018}, which uses experience replay (and was used as baseline by \citeauthor{Kaiser:2019}).

\paragraph{SimPLe}
\citet{Kaiser:2019} showed data-efficient learning is possible in Atari 2600 videos games from the arcade learning environment \citep{Bellemare:2013} with a purely model-based approach: only updating the policy with data sampled from a learnt parametric model $m$.  The resulting ``simulated policy learning'' (SimPLe) algorithm performed relatively well after just 102,400 interactions (409,600 frames --- two hours of simulated play) within each game.  In Algorithm \ref{alg:mbrl}, this corresponds to setting $K \times M = 16 \times 6400 =$ 102,400. Although SimPLe used limited data, it used a large number of samples from the model, similar to using $P =$ 800,000.\footnote{The actual number of reported model samples was $19 \times 800,000 = 15.2$ million, because $P$ was varied depending on the iteration.}

\paragraph{Rainbow DQN}
One of the main results by \citet{Kaiser:2019} was to compare SimPLe to Rainbow DQN \citep{Hessel:2018}, which combines the DQN algorithm \citep{Mnih:2013,Mnih:2015} with double Q-learning \citep{vanHasselt:2010,vanHasselt:2016}, dueling network architectures \citep{Wang:2016}, prioritised experience replay \citep{Schaul:2016}, noisy networks for exploration \citep{Fortunato:2017}, and distributional reinforcement learning \citep{Bellemare:2017}.
Like DQN, Rainbow DQN uses mini-batches of transitions sampled from experience replay \citep{Lin:1992} and uses Q-learning \citep{Watkins:1989} to learn the action-value estimates which determine the policy. Rainbow DQN uses multi-step returns \citep[cf.][]{Sutton:1988, SuttonBarto:2018} rather than the one-step return used in the original DQN algorithm.

\subsection{A data efficient Rainbow DQN}
In the notation of Algorithm \ref{alg:mbrl}, the total number of transitions sampled from replay during learning will be $K \times P$, while the total number of interactions with the environment will be $K \times M$. Originally, in both DQN and Rainbow DQN, a batch of 32 transitions was sampled every 4 real interactions. So $M=4$ and $P=32$.  The total number of interactions was 50M (200 million frames), which means $K = 50\mathrm{M}/4 = 12.5\mathrm{M}$.

In our experiments below, we trained Rainbow DQN for a total number of real interactions comparable to that of SimPLe, by setting $K=$ 100,000, $M=1$ and $P = 32$. The total number of replayed samples (3.2 million) is then less than the total number of model samples used in SimPLe (15.2 million). Rainbow DQN is more efficient computation-wise, since sampling from a replay buffer is faster than generating a transition with a learnt model.

The only other changes we made to make Rainbow DQN more data efficient were to increase the number of steps in the multi-step returns from 3 to 20, and to reduce the number of steps before we start sampling from replay from $20,000$ to $1600$. We used the fairly standard convolutional Q network from \citet{Hessel:2018IBDRL}.  We have not tried to exhaustively tune the algorithm and we do not doubt that the algorithm can be made even more data efficient by fine-tuning its hyper-parameters.

\subsection{Empirical results}
\begin{figure}
    \centering
    \includegraphics[width=0.9\linewidth]{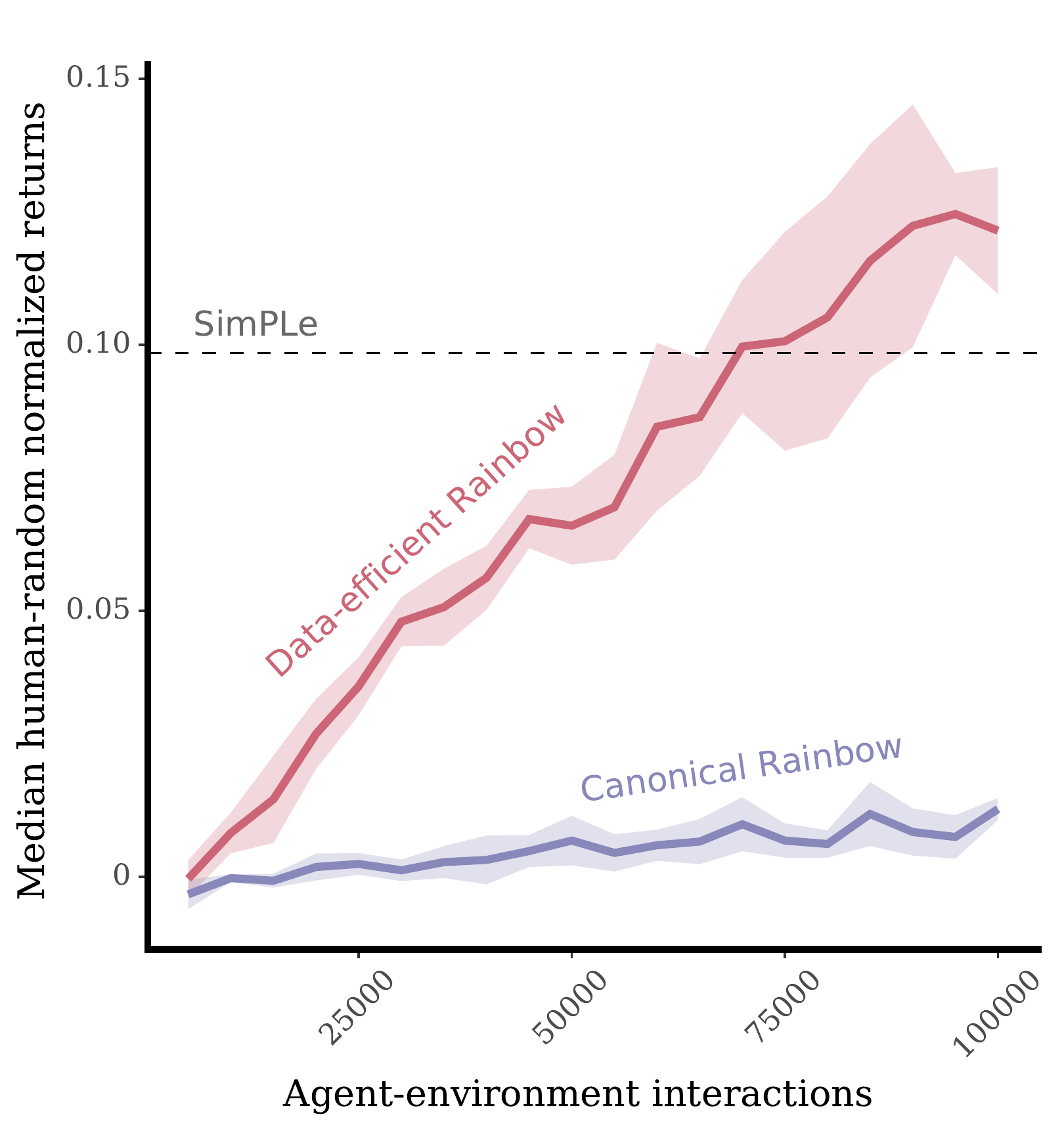}
    \caption{Median human-normalised episode returns of a tuned Rainbow, as a function of environment interactions ($=$frames/action repeats). The horizontal dashed line corresponds to the performance of SimPLe \citep{Kaiser:2019}. Error bars are computed over 5 seeds.}
    \label{aggregate}
\end{figure}

We ran Rainbow DQN on the same 26 Atari games reported by \citet{Kaiser:2019}.
In Figure \ref{aggregate}, we plotted the performance of our version of Rainbow DQN as a function of the number of interactions with the environment. Performance was measured in terms of episode returns, normalised using human and random scores \citep{vanHasselt:2016}, and then aggregated across the 26 games by taking their median. Error bars are shown as computed over the 5 independent replicas of each experiment. The final performance of SimPLe, according to the same metric, is shown in Figure \ref{aggregate} as a dashed horizontal line.

As expected, the hyper-parameters proposed by \citet{Hessel:2018} for the larger-data regime of 50 million interactions are not well suited to a regime of extreme data-efficiency (purple line in Figure \ref{aggregate}). Performance was better for our slightly-tweaked data-efficient version of Rainbow DQN (red line), that matched the performance of SimPLe after just 70,000 interactions with the environment, reaching roughly $25\%$ higher performance by 100,000 interactions. The performance of our agent was superior to that of SimPLe in 17 out of 26 games. More detailed results are included in the appendix, including ablations and per-game performance.

\section{Conclusions}
We discussed commonalities and differences between replay and model-based methods.  In particular, we discussed how model errors may cause issues when we use a parametric model in a replay-like setting, where we sample observed states from the past. We note that model-based learning can be unstable in theory, and hypothesised that replay is likely a better strategy under that state sampling distribution.  This is confirmed by at-scale experiments on Atari 2600 video games, where our replay-based agent attained state-of-the-art data efficiency, besting the impressive model-based results by \citet{Kaiser:2019}.

We further hypothesised that parametric models are perhaps more useful when used either 1) to plan backward for credit assignment, or 2) to plan forward for behaviour.  Planning forward for credit assignment was hypothesised and shown to be less effective, even though the approach seems quite common. The intuitive reasoning was that when the model is inaccurate, then planning backwards with a learnt model may lead to updating fictional states, which seems less harmful than updating real states with inaccurate transitions as would happen in forward planning for credit assignment.  Forward planning for \emph{behaviour}, rather than credit assignment, was deemed potentially useful and less likely to be harmful for learning, because the resulting plan is not trusted as real experience by the prediction or policy updates. Empirical results supported these conclusions.

There is a rich literature on model-based reinforcement learning, and this paper cannot cover all the potential ways to plan with learnt models.  One notable topic that is out of scope for this paper is the consideration of abstract models \citep{Silver:2017} and alternative ways to use these models in addition to classic planning \citep[cf.][]{Weber:2017}.

Finally, we note that our discussion focused mostly on the distinction between parametric models and replay, because these are the most common, but it is good to acknowledge that one can also consider \emph{non-parametric} models.  For instance, one could apply a nearest-neighbours or kernel approach to a replay buffer, and thereby obtain a non-parametric model that can be equivalent to replay when sampled at the observed states, but that can interpolate and generalise to unseen states when sampled at other states \citep{Pan:2018}.  This is conceptually an appealing alternative, although it comes with practical algorithmic questions of how best to define distance metrics in high-dimensional state spaces. This seems another interesting potential avenue for more future work.

\subsubsection*{Acknowledgments}
The authors benefitted greatly from feedback from Tom Schaul, Adam White, Brian Tanner, Richard Sutton, Theophane Weber, Arthur Guez, and Lars Buesing.

\small

\bibliography{references}
\bibliographystyle{abbrvnat}

\normalsize

\newpage
\onecolumn
\centerline{\huge{\textbf{Appendix}}}
\appendix
\section{Divergence example}
As a concrete illustration of the issue discussed in Section 3, consider the two-state Markov reward process (MRP) depicted in Figure \ref{mrp}a. This example is similar in nature to other examples from the literature [Baird, 1995, Tsitsiklis and Van Roy, 1997]. On each transition with probability $p$ we transition to state $s=1$ (left), and with probability of $1 - p$ transition to $s=2$ (right). All rewards are 0, discount is $\gamma = 0.99$. Each state has a single feature $x(s) = s$. The goal is to learn a weight $w$ such that $v_{\w}(s) = w \times x(s)$ is accurate. The optimal weight is, trivially, $w = 0$.  

As discussed, the expected update can diverge if the sampling distribution of states $d$ does not match the sampling distribution under model $m$.
Figure \ref{mrp}b shows under which sampling probabilities $d(s) = P(S_t = s)$ and transition probabilities $P(S_{t+1}=1 | S_t)$ the updates diverge.
Divergence occurs when the probability of sampling state $s=1$ (under $d$) is sufficiently higher than the transition probability into state $s=1$. Note how oversampling state $s=2$ is less harmful for this specific choice of function approximation.

Updates do not diverge because the learnt model is inaccurate, but because of a mismatch between the model dynamics and the state sampling distribution. Divergence can thus occur even when using the true dynamics, if $d$ does not match the steady-state distribution induced by such dynamics. For a true dynamics of $p(S_{t+1}=1|S_t)=0.5$, Figure \ref{mrp}c shows the likelihood of observing divergence as a function of the number of samples used to estimate the empirical distribution $d$, assuming a perfect model and unbiased data-dependent estimates of $d$.

\section{Experiment details: Scalability of planning}

The layout of maze used in these experiments is shown in the main text. The agent can see a $5\times5$ portion of the maze, centered in its current location, where walls are encoded with 1s, and free cells as 0s. The agent can choose among 4 actions (up, down, left, right) that result in deterministic transitions to the adjacent cell, as long as such cell is empty; if the cell is a wall, the action has no effect. 

Both the forward Dyna agent and the replay-based Q-learning agent used a multi-layer perceptron (with two fully connected hidden layers, of size 20, and ReLU activations throughout) to approximate Q-values. The final output layer had no activation, and had only 4 nodes, one per action. The forward Dyna agent used separate networks with the same hidden layers to model state transitions, rewards and terminations; the output layers of these had $25$, $1$, and $1$ outputs. Both agents use a replay with a capacity of $10000$ transitions; the Q-networks are updated  with double Q-learning, on mini-batches of size 32; updates are rescaled by TensorFlow's implementation of the Adam optimizer, using a learning rate of $1e-3$. In the replay-based agent the update is computed using only the real data from the transition $s_{t-1}, a_{t-1}, r_t, \gamma_t, s_t$; in forward dyna the fictional transition $s_{t-1}, a_{t-1}, m_R(s_{t-1}, a_{t-1}), m_T(s_{t-1}, a_{t-1}), m_S(s_{t-1}, a_{t-1})$ is used instead, where $m_R, m_T, m_S$ are the outputs of the three neural networks used to parameterize the model.

\section{Experiment details: Benefits of Planning}

For these experiments we run on the four-rooms environment shown in the text. At the beginning of each episode, the agent's starting position is randomized and the goal position is held fixed. The dynamics are deterministic, with four actions that move the agent in the four cardinal directions, and a no-op action. The state is fully-observed, and we use a tabular (state-index) representation for these experiments. In both experiments we learn an exact Bayesian tabular model. We also learn a tabular value function in tandem using one-step (tabular) Q-learning.

\begin{figure*}[t]
    \centering
    \begin{tabular}{ccc}
    (a) & \hspace{0.5cm}(b) & \hspace{0.5cm}(c) \\
    \raisebox{0.4\height}{\includegraphics[width=0.26\linewidth]{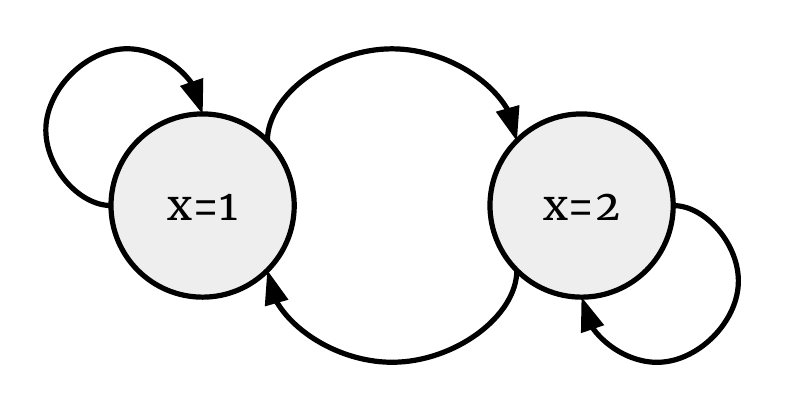}}
    &
    \includegraphics[width=0.32\linewidth]{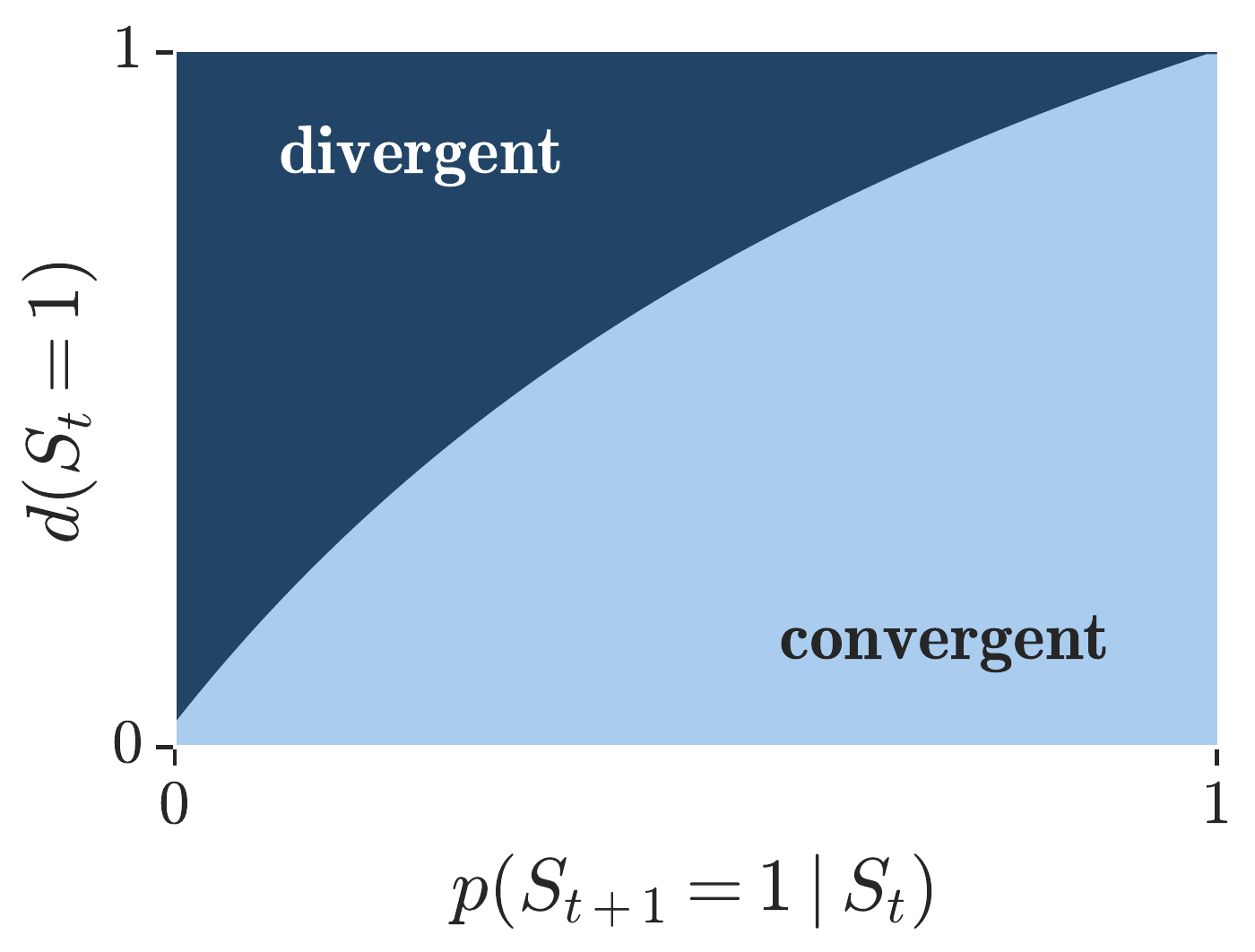}
    &
    \includegraphics[width=0.32\linewidth]{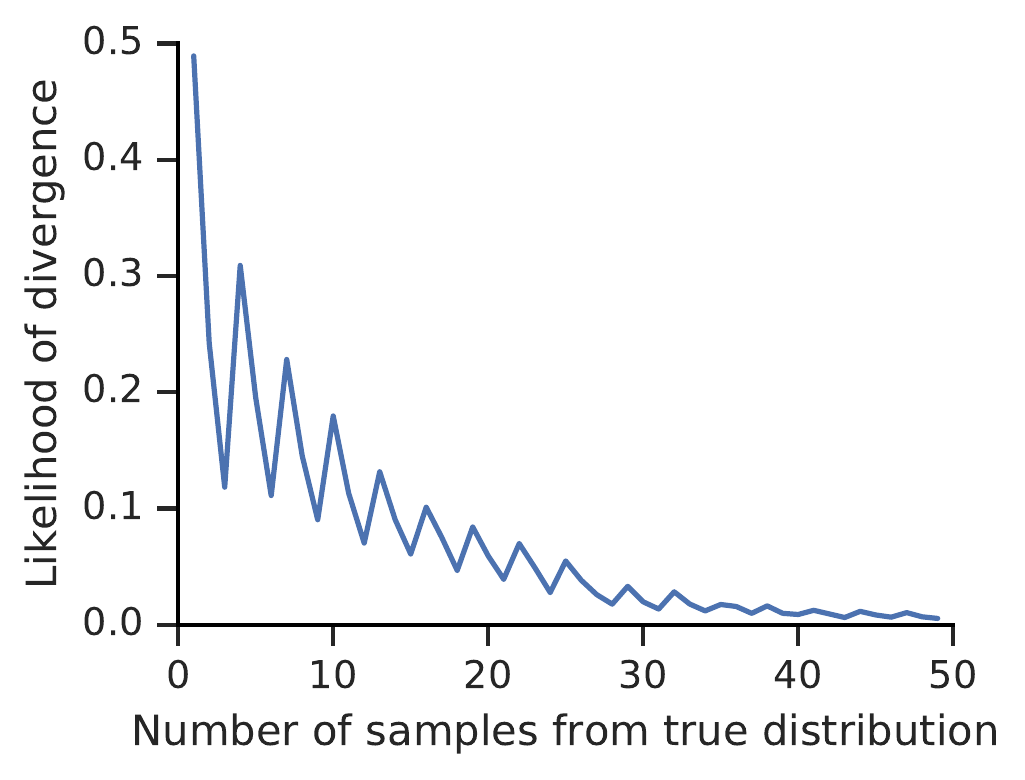}
    \end{tabular}
    \caption{
    (a) A simple Markov reward process from \citet{Sutton:2016}.
    (b) Observed divergence for different sampling distributions $d$ and transition probabilities $p$.
    (c) Assuming a perfect model and an unbiased data-dependent estimate of $d$ sampled from an instantiation of the environment with $p(S_{t+1}=1|S_t)=0.5$, we plot the likelihood of observing divergence as a function of the number of samples used to estimate the empirical distribution $d$. \label{mrp} }
    \includegraphics[width=\linewidth]{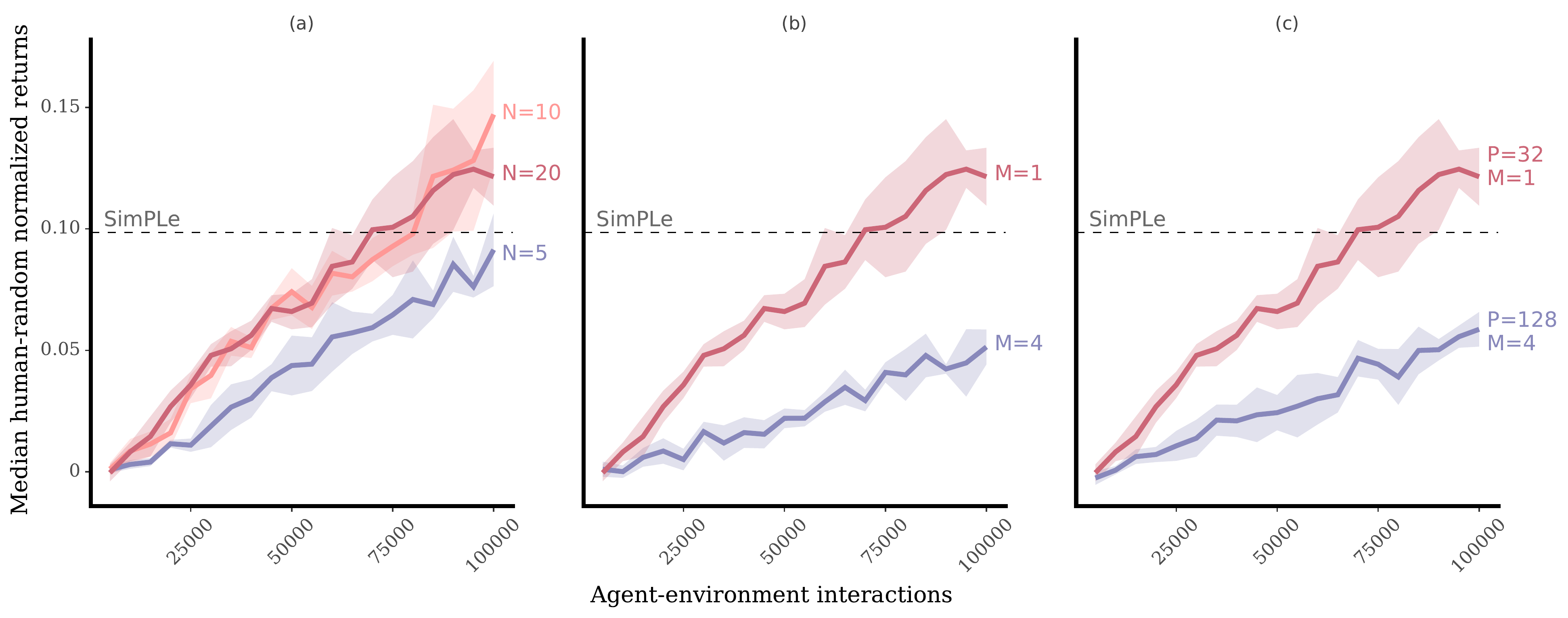}
    \caption{\textbf{Left}: an ablation experiment where we investigate the effect of various settings for the length of the multi-step bootstrapped targets. \textbf{Center}: an ablation experiment where we compare our variant of Rainbow to performing updates every 4 steps as in the canonical Rainbow DQN. \textbf{Center} comparing our data efficient Rainbow DQN with $M=1, P=32$ to a different Rainbow DQN which achieves the same $4 \times$ increase in the number of transitions sampled from replay, by increasing the batch size instead ($M=4, P=128$).  \label{ablations}}
\end{figure*}

\section{Additional results on Atari}

In Figure \ref{ablations}a and \ref{ablations}b we show the results of ablation experiments performed to isolate the effect of increasing the bootstrapping parameter and the effect of increasing the frequency of updates.
In Figure \ref{ablations}a we show the effect of varying the bootstrapping parameter $N \in [5,10,20]$, while keeping the update frequency fixed ($M=1$). Consistently with our expectations a bootstrapping length of $N=5$ resulted in much worse performance, although our such variant of Rainbow DQN still achieved results comparable to those of SimPLe. Both $N=10$ and $N=20$ resulted in good performance, with the difference between the two not found to be statistically significant (under a Welch's test applied to the 5 replicas of each hyper-parameter evaluation, with significance level of 0.1). 
In Figure \ref{ablations}b we show the effect of varying the frequency of the updates $M \in [1, 4]$, while keeping the number of steps before bootstrapping fixed ($N=20$); The agent which performs updates on each step performed much better, and the gap in performance was larger then the gap observed when varying the bootstrap parameter $N$.
Finally, in Figure \ref{ablations}c we report an additional experiment where we compare our variant of Rainbow DQN with $M=1, P=32$ to a different Rainbow DQN which achieves the same $4 \times$ increase in the number of transitions sampled from replay, by increasing the batch size instead ($M=4, P=128$). The performance was significantly lower.

\section{Table of results}
In Table \ref{table1} we report, for each of the 26 Atari game used by Kaiser et al. [2019] in their experiments, the mean episode return, at the end of training, of both SimPLe and our data efficient variant of Rainbow. On each score we mark in bold the best performing among the two agents. We also report the reference human and random scores that were used to normalize the scores in all learning curves. 

\vspace{10pt}
{
\renewcommand{\arraystretch}{1.5}
\begin{table}[h!]
\centering
\begin{tabular}{lrrrr}
\toprule
            Game &    Human &   Random &   SimPLe &  Rainbow \\
\midrule
           alien &   7127.7 &    227.8 &    405.2 &    \textbf{739.9} \\
          amidar &   1719.5 &      5.8 &     88.0 &    \textbf{188.6} \\
         assault &    742.0 &    222.4 &    369.3 &    \textbf{431.2} \\
         asterix &   8503.3 &    210.0 &   \textbf{1089.5} &    470.8 \\
      bank\_heist &    753.1 &     14.2 &      8.2 &     \textbf{51.0} \\
     battle\_zone &  37187.5 &   2360.0 &   5184.4 &  \textbf{10124.6} \\
          boxing &     12.1 &      0.1 &    \textbf{9.1} &      0.2 \\
        breakout &     30.5 &      1.7 &    \textbf{12.7} &      1.9 \\
 chopper\_command &   7387.8 &    811.0 &   \textbf{1246.9} &    861.8 \\
   crazy\_climber &  35829.4 &  10780.5 &  \textbf{39827.8} &  16185.3 \\
    demon\_attack &   1971.0 &    152.1 &    169.5 &    \textbf{508.0} \\
         freeway &     29.6 &      0.0 &     20.3 &     \textbf{27.9} \\
       frostbite &   4334.7 &     65.2 &    254.7 &    \textbf{866.8} \\
          gopher &   2412.5 &    257.6 &    \textbf{771.0} &    349.5 \\
            hero &  30826.4 &   1027.0 &   1295.1 &   \textbf{6857.0} \\
       jamesbond &    302.8 &     29.0 &    125.3 &    \textbf{301.6} \\
        kangaroo &   3035.0 &     52.0 &    323.1 &    \textbf{779.3} \\
           krull &   2665.5 &   1598.0 &   \textbf{4539.9} &   2851.5 \\
  kung\_fu\_master &  22736.3 &    258.5 &  \textbf{17257.2} &  14346.1 \\
       ms\_pacman &   6951.6 &    307.3 &    762.8 &   \textbf{1204.1} \\
            pong &     14.6 &    -20.7 &      \textbf{5.2} &    -19.3 \\
     private\_eye &  69571.3 &     24.9 &     58.3 &    \textbf{97.8} \\
           qbert &  13455.0 &    163.9 &    559.8 &   \textbf{1152.9} \\
     road\_runner &   7845.0 &     11.5 &   5169.4 &   \textbf{9600.0} \\
        seaquest &  42054.7 &     68.4 &    \textbf{370.9} &    354.1 \\
       up\_n\_down &  11693.2 &    533.4 &   2152.6 &   \textbf{2877.4} \\
\bottomrule
\vspace{2pt}
\end{tabular}
\caption{Mean episode returns of Human, Random, SimPLe and Rainbow agents, on each of 26 Atari games. The Rainbow results are measured at the end of training and averaged across 5 seeds; the results for SimPLe are taken from Kaiser et al. [2019]. On each game we mark as bold the higher score among SimPLe and Rainbow.}
\label{table1}
\end{table}
}

\newpage

\section{Atari hyper-parameters}
In Table \ref{table2} we report, for completeness and ease of reproducibility, the hyper-parameter settings used by the canonical Rainbow DQN agent, as well as the hyper-parameters that differ in our data efficient variation.

Several parameters that are common between the canonical and data-efficient variants of the algorithm may have large effects on data efficiency. However, our main goal was to do a clean comparison, rather than to push for maximal performance, and therefore we have made no effort fine-tuning these.

\vspace{4pt}
{\renewcommand{\arraystretch}{1.1}
\begin{table}[h]
\small
\centering
\begin{tabular}{lrr}
\toprule
Hyper-parameter & \multicolumn{2}{r}{setting (for both variations)} \\
\midrule
Grey-scaling && True \\
Observation down-sampling && (84, 84) \\
Frames stacked && 4 \\
Action repetitions && 4 \\
Reward clipping && [-1, 1] \\
Terminal on loss of life && True \\
Max frames per episode && 108K \\
Update & \multicolumn{2}{r}{Distributional Double Q} \\
Target network update period${}^{*}$ & \multicolumn{2}{r}{every 2000 updates} \\
Support of Q-distribution && 51 bins \\
Discount factor && 0.99 \\
Minibatch size && 32 \\
Optimizer && Adam \\
Optimizer: first moment decay && 0.9 \\
Optimizer: second moment decay && 0.999 \\
Optimizer: $\epsilon$ && $0.00015$ \\
Max gradient norm && 10 \\
Priority exponent && 0.5 \\
Priority correction${}^{**}$ && 0.4 $\rightarrow$ 1\\
Hardware && CPU \\
Noisy nets parameter && 0.1 \\
\midrule
Hyper-parameter & canonical & data-efficient\\
\midrule
Training frames & 200,000,000 & 400,000 \\
Min replay size for sampling & 20,000 & 1600 \\
Memory size & 1,000,000 steps & unbounded \\
Replay period every & 4 steps & 1 steps \\
Multi-step return length & 3 & 20 \\
Q network: channels & 32, 64, 64 & 32, 64 \\
Q network: filter size & $8\times8, 4\times4, 3\times3$ & 5 × 5, 5 × 5 \\
Q network: stride & 4, 2, 1 & 5, 5 \\
Q network: hidden units & 512 & 256 \\
Optimizer: learning rate & 0.0000625 & 0.0001 \\
\bottomrule
\multicolumn{3}{p{12cm}}{\footnotesize{${}^{*}$ The target network update period depends on the number of updates (not frames). This means that this update is more frequent in the data-efficient variant, in terms of frames.}}\\
\multicolumn{3}{p{12cm}}{\footnotesize{${}^{**}$ The priority correction linearly annealed from 0.4 to 1 during training: $\text{exponent} = (1 - \eta) \times 0.4 + \eta \times 1.0$, where $\eta = \text{current\_step}/\text{max\_steps}$. For the canonical variant, $\text{max\_step} = 50\text{M}$, for the data-efficient variant $\text{max\_step} = 100\text{K}$}}\\
\end{tabular}
\vspace{2pt}
\caption{The hyper-parameters used by the canonical and the data-efficient variant of the Rainbow DQN agent.}
\label{table2}
\end{table}
}

\end{document}